\documentclass{aamas2017}

\usepackage[utf8]{inputenc}
\usepackage{graphicx}
\usepackage{amsmath}
\usepackage{bm}
\usepackage{wrapfig}
\usepackage{amssymb}
\usepackage{epstopdf}
\usepackage{mathtools}
\usepackage{url}
\usepackage{floatrow}
\usepackage[export]{adjustbox}
\usepackage{indentfirst}
\usepackage{pgf}
\usepackage{pgfplotstable}
\usepackage{amsmath,amssymb}
\usepackage[Algorithm]{algorithm}
\usepackage[]{algpseudocode}
\usepackage{tabularx}
\usepackage{bm}
\pgfplotsset{compat=1.12}
\usepackage{pgfplots}

\makeatletter
\newcommand{\StatexIndent}[1][3]{%
  \setlength\@tempdima{\algorithmicindent}%
  \Statex\hskip\dimexpr#1\@tempdima\relax}
\algdef{S}[IF]{Ifcon}[1]{\algorithmicif\ #1}%
\makeatother

\newcolumntype{P}[1]{>{\centering\arraybackslash}p{#1}} 

\def\SynTeam{SynTeam}

\newtheorem{mydef}{Definition}
\newtheorem{proposition}{Proposition}
\makeatletter
\def\BState{\State\hskip-\ALG@thistlm}
\makeatother

\begin{document}


\title{Synergistic Team Composition}



%
%
%
%

%

\numberofauthors{2}
\author{
\alignauthor Ewa Andrejczuk\\
       \affaddr{Artificial Intelligence Research Institute (IIIA-CSIC)}\\
       \affaddr{Change Management Tool S.L}\\
       \affaddr{Barcelona, Spain}\\
       \email{ewa@iiia.csic.es}
\alignauthor Juan A. Rodr\'{\i}guez-Aguilar \\
       \affaddr{Artificial Intelligence Research Institute (IIIA-CSIC)}\\
       \affaddr{Barcelona, Spain}\\
       \email{jar@iiia.csic.es}
\and  
\alignauthor Carme Roig\\
       \affaddr{Institut Torras i Bages}\\
       \affaddr{L'Hospitalet de Llobregat, Spain}\\
       \email{mroig112@xtec.cat}
\alignauthor Carles Sierra\\
       \affaddr{Artificial Intelligence Research Institute (IIIA-CSIC)}\\
       \affaddr{Barcelona, Spain}\\
       \email{sierra@iiia.csic.es}
}

\maketitle

\begin{abstract}
Effective teams are crucial for organisations, especially in environments that require teams to be constantly created and dismantled, such as software development, scientific experiments, crowd-sourcing, or the classroom. Key factors influencing team performance are competences and personality of team members. Hence, we present a computational model to compose proficient and congenial teams based on individuals' personalities and their competences to perform tasks of different nature. With this purpose, we extend Wilde's post-Jungian method for team composition, which solely employs individuals’ personalities. The aim of this study is to create a model to partition agents into teams that are balanced in competences, personality and gender. Finally, we present some preliminary empirical results that we obtained when analysing student performance. Results show the benefits of a more informed team composition that exploits individuals' competences besides information about their personalities.
\end{abstract}





\section{Introduction}

Some tasks, due to their complexity, cannot be carried out by single individuals. They need the concourse of sets of people composing teams. Teams provide a structure and means of bringing together people with a suitable mix of individual properties (such as competences or personality). This can encourage the exchange of ideas, their creativity, their motivation and job satisfaction and can actually extend individual capabilities. In turn, a suitable team can improve the overall productivity, and the quality of the performed tasks. However, sometimes teams work less effectively than initially expected due to several reasons: a bad balance of their capacities, incorrect team dynamics, lack of communication, or difficult social situations. Team composition is thus a problem that has attracted the interest of research groups all over the world, also in the area of multiagent systems. MAS research has widely acknowledged competences as important for performing tasks of different nature \cite{Anagnostopoulos12onlineteam,Chen2015,Okimoto,Rangapuram2015}. However, the majority of the approaches represent capabilities of agents in a Boolean way (i.e., an agent either has a required skill or not). This is a simplistic way to model an agent's set of capabilities as it ignores any skill degree. In real life, capabilities are not binary since every individual (e.g. human or software) shows different performances for each competence. Additionally, the MAS literature has typically disregarded significant organizational psychology findings (with the exception of several recent, preliminary attempts like \cite{FarhangianPPS15} or \cite{alberola2016artificial}). Numerous studies in organizational psychology \cite{Arnold,Mount,White} underline the importance of personality traits or \emph{types} for team composition. Other studies have focused on how team members should differ or converge in their characteristics, such as experience, personality, level of skill, or gender, among others \cite{West}, in order to increase performance. 

In this paper, we focus on scenarios where a complex task requires the collaboration of individuals within a team. More precisely, we consider a scenario, where there are \emph{multiple instances of the same complex task}. The task has a task type and a set of competence requests with competence levels needed to solve the task. We have a pool of human agents characterized by gender, personality, and a set of competences with competence levels. 
Our goal is to partition agents into teams so that within a task all competence requirements are covered (whenever possible) and team members work well together. That is, each resulting team is both \emph{proficient} (covers the required competences) and \emph{congenial} (balances gender and psychological traits). We refer to these teams as \emph{synergistic teams}. We define the \emph{synergistic value} of a team as its balance in terms of competence, personality and gender. Each synergistic team works on the very same task. This scenario is present in many real-life settings, for instance a classroom or a crowdsourcing task. 
With this purpose, we design an algorithm that uses a greedy technique both to match competences with the required ones and at the same time to balance the psychological traits of teams' members. 

This paper makes the following contributions. To start with, we formalise the synergistic team formation problem as the problem of partitioning a group of individuals into teams with limited size. 
We provide an approximate local algorithm to solve  the team composition problem. We empirically evaluate the algorithm using real data. Preliminary results show that our algorithm predicts better the performance of teams than the experts that know students' social situation, background and competences. 

\textbf{Outline.} The remaining of this paper is structured as follows. Section~\ref{related} opens with an overview of the related work. Section~\ref{pers} gives the personality background for our model. Section~\ref{sec:model} describes the synergistic team composition problem and Section~\ref{sec:TeamForm} presents our algorithm to solve the synergistic team composition problem. Then, Section~\ref{sec:results} presents results of our algorithm in the context of team composition in the classroom. Finally, Section~\ref{sec:discuss} discusses our approach and future work.
\vspace{-2mm}
\section{Background} \label{related}
To the best of our knowledge, \cite{farhangian2015agent} is the only model that considers both personality and competences while composing teams.  There, the influence of personality on different task allocation strategies (minimizing either undercompetence or overcompetence) is studied. Henceforth, this work is the most relevant for us, however there are substantial differences between our work and \cite{farhangian2015agent}. Firstly, authors do not propose an algorithm to compose teams  based on \emph{both} personality and competences. Secondly, gender balance is not considered in their setting. Finally, \cite{farhangian2015agent} does not provide an evaluation involving real data (only an agent-based simulation is presented).

The rest of the literature relevant to this article is divided into two categories as proposed in \cite{andrejczuk}: those that consider agent capacities (individual and social capabilities of agents) and those that deal with agent personality (individual behaviour models).

\textbf{Capacity.}
The capacity dimension has been exploited by numerous previous works \cite{Anagnostopoulos12onlineteam,Chalkiadakis2012,Chen2015,Crawford,Liemhetcharat2014,Okimoto,JAR2015,Rangapuram2015}. In contrast to our work, where the competences are graded, in the majority of works agents are assumed to have multiple binary skills (i.e., the agent either has a skill or not). For instance, \cite{Okimoto,Crawford} use agents' capabilities to compose one k-robust team for a single task. A team is $k$-robust if removing any $k$ members from the team does not affect the completion of the task. \cite{Anagnostopoulos12onlineteam} uses competences and communication cost in a context where tasks sequentially arrive and teams have to be composed to perform them. Each task requires a specific set of competences and the team composition algorithm is such that the workload per agent is fair across teams. 

\textbf{Personality.}
In the team formation literature, the only two models to our knowledge considering personality to compose teams are \cite{FarhangianPPS15} and \cite{alberola2016artificial}. \cite{alberola2016artificial} uses Belbin theory to obtain human predominant \emph{roles} (we discuss this method in Section \ref{pers}).  Additionally, the gender is not taken into account while composing heterogeneous teams, which we believe may be important for team congeniality. Regarding \cite{FarhangianPPS15}, Farhangian et al. use the classical MBTI personality test (this method is discussed in Section \ref{pers}). They look for the best possible team built around a selected leader. In other words, the \emph{best} team for a particular task is composed. Gender balance is not considered in this setting. Finally, although \cite{FarhangianPPS15}'s team composition considered real data, the resulting teams' performance was not validated in any real setting (Bayesian theory was used to predict the probability of success in various team composition conditions).
\vspace{-3mm}
\section{Personality} \label{pers}
In this section, we discuss the most prominent approaches to measure human personality and we explain the details of the method we have decided to examine.

Personality determines people's behaviour, cognition and emotion. Different personality theorists present their own definitions of personality and different ways to measure it based on their theoretical positions.  

The most popular approach is to determine personality through a set of questions. There have been several simplified schemes developed over the years to profile human personality. The most populars are:
\begin{enumerate}
\vspace{-1.5mm}
\item the Five Factor Model (aka FFM or ``Big Five''), which uses five broad dimensions to describe human personality \cite{Costa};
\vspace{-1.5mm}
\item Belbin theory \cite{belbin}, which provides a theory on how different role types influence teamwork; and 
\vspace{-1.5mm}
\item the Myers-Briggs Type Indicator (MBTI) scheme designed to indicate psychological preferences in how people perceive the world and make decisions \cite{Myers}. 
\end{enumerate}
\vspace{-1.5mm}
According to \cite{Poropat}, FFM personality instruments fail to detect significant sex differences in personality structures. It is also argued that the Big Five dimensions are too broad and heterogeneous, and lack the specificity to make accurate predictions in many real-life settings \cite{Boyle,johnson2004genetic}. 

Regarding Belbin theory, the results of previous studies considering the correlation between team composition and team performance are ambiguous. Even though some research shows weak support or does not show support for this theory at all \cite{batenburg2013belbin,van2008belbin,partington1999belbin}, it remains popular.

Finally, the MBTI measure consists of four dimensions on a binary scale (e.g. either the person is Extrovert or Introvert). Within this approach, every person falls into one of the sixteen possible combinations of the four letter codes, one letter representing one dimension. This approach is easy to interpret by non-psychologists, though reliance on dichotomous preference scores rather than continuous scores excessively restricts the level of statistical analysis \cite{devito}.

Having considered the arguments above,  we have decided to explore a novel method: the Post-Jungian Personality Theory, which is a modified version of the Myers-Briggs Type Indicator (MBTI) \cite{Myers}, the ``Step II'' version of Quenk, Hammer and Majors \cite{Wilde2013}. The questionnaire to determine personality is short, contains only 20 quick questions (compared to the 93 MBTI questions). This is very convenient for both experts wanting to design teams and individuals doing the test since completing the test takes just a few minutes (for details of the questionnaire, see \cite[p.21]{Wilde2013}). Douglass J. Wilde claims that it covers the same psychological territory as MBTI \cite{Wilde2009}. In contrast to the MBTI measure, which consists of four binary dimensions, the Post-Jungian Personality Theory uses the \emph{numerical} data collected using the questionnaire \cite{Wilde2011}. The results of this method seem promising, since  within a decade this novel approach has tripled the fraction of Stanford teams awarded national prizes by the Lincoln Foundation \cite{Wilde2009}.

The test is based on the pioneering psychiatrist Carl Gustav Jung's cognitive-mode personality model \cite{PT}. It has two sets of variable pairs called psychological functions: 
\vspace{-1.5mm}
\begin{itemize}
\item {\bf Sensing / Intuition (SN)} --- describes the way of approaching problems
\vspace{-1.5mm}
\item {\bf Thinking / Feeling (TF)} --- describes the way of making decisions
\end{itemize} 
\vspace{-1.5mm}
and two sets of psychological attitudes:
\vspace{-1.5mm}
\begin{itemize}
\item {\bf Perception / Judgment (PJ)} --- describes the way of living
\vspace{-1.5mm}
\item {\bf Extroversion / Introversion (EI)} --- describes the way of interacting with the world
\end{itemize} 
\vspace{-1.5mm}
For instance, for the Feeling-Thinking (TF) dimension, a value between -1 and 0 means that a person is of the feeling type, and a value between 0 and 1 means she is of the thinking type. Psychological functions and psychological attitudes compose together a personality. Every dimension of a personality (EI, SN, TF, PJ) is tested by five multiple choice true/false questions.
\vspace{-2mm}
\section{Team Composition Model}\label{sec:model}

In this section we introduce and formalise our team composition problem. First, section \ref{ssec:basic} introduces the basic notions of agent, personality, competence, and team, upon which we formalise our problem. Next, we formalise the notion of task assignment for a single team and a single task, and we characterise different types of assignments. Sections \ref{ssec:proficiency} and \ref{ssec:congeniality} show how to evaluate the proficiency and congeniality degrees of a team. Based on these measures, in section \ref{ssec:synergisticProblem} we formalise the \emph{synergistic team composition problem}.
\subsection{Basic definitions} 
\label{ssec:basic}

In our model, we consider that each agent is a human. We characterise each agent by the following properties:
\begin{itemize}
\vspace{-1.5mm}
\item A unique \emph{identifier} that distinguishes an agent from others (e.g. ID card number, passport number, employee ID, or student ID).
\vspace{-1.5mm}
\item \emph{Gender.} Human agents are either a man or a woman.
\item A \emph{personality} represented by four personality traits. Each personality trait is a number between -1 and 1.  
\item A \emph{set of competences}. A competence integrates knowledge, skills, personal values, and attitudes that enable an agent to act correctly in a job, task or situation \cite{roe2002competences}. Each agent is assumed to possess a set of competences with associated competence levels. This set may vary over time as an agent evolves. 
\end{itemize}
\vspace{-1.5mm}
Next, we formalise the above-introduced concepts.
\vspace{-1.5mm}
\begin{mydef}
A \emph{personality profile} is a vector $\langle sn, \mathit{tf}, ei, pj \rangle \in [-1, 1]^4$, where each $sn, \mathit{tf}, ei, pj$ represents one personality trait.
\end{mydef}

We denote by $C = \{c_1, \dots , c_m\}$ the whole set of competences, where each element $c_i \in C$ stands for a competence.

\begin{mydef}
A \emph{human agent} is represented as a tuple $\langle id, g, \emph{{\bf p}}, l \rangle$ such that:
\begin{itemize}
\item $id$ is the agent's identifier;
\item $g \in \{man, {\mathit woman}\}$ stands for their gender;
\item $\emph{\bf{p}}$ is a personality profile vector $\langle sn, \mathit{tf}, ei, pj \rangle \in [-1, 1]^4$;
\item $l: C \to{[0,1]}$  is a function that assigns the probability that the agent will successfully show competence $c$. We will refer to $l(c)$ as the \emph{competence level} of the agent for competence $c$. We assume that when an agent does not have a competence (or we do not know about it), the level of this competence is zero.
\end{itemize}
\end{mydef}

Henceforth, we will note the set of agents as $A =\{a_1,\ldots, \linebreak a_n\}$. Moreover, We will use super-indexes to refer to agents' components. For instance, given an agent $a \in A$, $id^{a}$ will refer to the $id$ component of agent $a$. We will employ matrix $L \in [0,1]^{n \times m}$ to represent the competence levels for each agent and each competence.
\vspace{-2mm}
\begin{mydef}[Team] A \emph{team} is any non-empty subset of $A$ with at least two agents. We denote by $\cal{K_A}$ $ = (2^A \setminus \{\emptyset\})\setminus \{\{a_i\}| a_i \in A\}$ the set of all possible teams in $A$. 
\end{mydef}
\vspace{-2mm}
We assume that agents in teams coordinate their activities for mutual benefit. 

\subsection{The task assignment problem} 
\label{ssec:assignment}

In this section we focus on how to assign a team to a task.
A task type determines the competence levels required for the task as well as the importance of each competence with respect to the others. For instance, some tasks may require a high level of creativity because they were never performed before (so there are no qualified agents in this matter). Others may require a highly skilled team with a high degree of coordination and teamwork (as it is the case for rescue teams). Therefore, we define a task type as:
\begin{mydef}
A task type $\tau$ is defined as a tuple \\ $\langle \lambda, \mu, {\{(c_{i},l_{i}, w_{i})\}_{i \in I_{\tau}}} \rangle$ such that:
\begin{itemize}
\item $\lambda \in [0,1]$ importance given to proficiency;
\item $\mu \in [-1,1]$ importance given to congeniality;
\item $c_{i} \in C$ is a competence required to perform the task;
\item $l_{i} \in [0,1]$ is the required competence level for competence $c_i$; 
\vspace{-1.5mm}
\item $w_{i} \in [0,1]$ is the importance of competence $c_i$ for the success of task of type $\tau$; and
\vspace{-1.5mm}
\item $\sum_{i \in I_{\tau}} w_i = 1$.
\end{itemize}
\end{mydef}
We will discuss the meaning of $\lambda$ and $\mu$  further ahead when defining synergistic team composition (see subsection \ref{ssec:synergisticProblem}).
Then, we define a task as:
\vspace{-1.5mm}
\begin{mydef}A \emph{task} $t$ is a tuple $\langle \tau, m \rangle$ such that $\tau$ is a task type and $m$ is the required number of agents, where $m\geq 2$.
\end{mydef}

Henceforth, we denote by $T$ the set of tasks and by $\mathcal{T}$ the set of task types. Moreover, we will note as $C_{\tau} =\{c_{i} | i \in I_{\tau}\}$ the set of competences required by task type $\tau$.

Given a team and a task type, we must consider how to assign competences to team members (agents). Our first, weak notion of task assignment only considers that all competences in a task type are assigned to some agent(s) in the team:

\begin{mydef}Given a task type $\tau$ and a team $K \in \cal{}K_A$, an assignment is a function $\eta: K \to 2^{C_{\tau}}$ satisfying that 
$C_{\tau} \subseteq \bigcup_{a \in K} \eta(a)$. 
\end{mydef}

\subsection{Evaluating team proficiency} \label{ssec:prof} 
\label{ssec:proficiency}

Given a task assignment for a team, next we will measure the \emph{degree of competence} of the team as a whole. This measure will combine both the degree of under-competence and the degree of over-competence, which we formally define first. Before that, we must formally identify the agents that are assigned to each competence as follows.
\vspace{-1.5mm}
\begin{mydef}
Given a task type $\tau$, a team $K$, and an assignment $\eta$, the set $\delta(c_{i}) = \{a \in K | c_{i} \in \eta(a)\}$ stands for the agents assigned to cover competence $c_{i}$.
\end{mydef}
\vspace{-1.5mm}
Now we are ready to define the degrees of undercompentence and overcompetence. 
\vspace{-1.5mm}
\begin{mydef}[Degree of undercompentence] \item
\vspace{-1.6mm}
Given a task type $\tau$, a team $K$, and an assignment $\eta$, we define the degree of undercompetence of the team for the task as:
\vspace{-2.5mm}
\begin{equation*}
u(\eta)=
\sum_{i \in I_{\tau}} w_{i} \cdot \frac{\sum_{a \in \delta(c_{i})} |\min(l^{a}(c_{i}) - l_{i},0)|}{|\{a \in \delta(c_{i})|l^{a}(c_{i})-l_{i} < 0\}|}
\end{equation*}
\end{mydef}
\vspace{-2.5mm}
\begin{mydef}[Degree of overcompetence] \item
\vspace{-1.6mm}
Given a task type $\tau$, a team $K$, and an assignment $\eta$, we define the degree of overcompetence of the team for the task as:
\vspace{-2.5mm}
\begin{equation*}
o(\eta)=
\sum_{i \in I_{\tau}} w_i \cdot \frac{\sum_{a \in \delta(c_{i})} \max(l^{a}(c_{i}) - l_{i},0)}{|\{a \in \delta(c_{i})|l^{a}(c_{i})-l_{i} > 0\}|}
\end{equation*}
\end{mydef}
\vspace{-1.5mm}
Given a task assignment for a team, we can calculate its competence degree to perform the task by combining its overcompetence and undercompetence as follows. 
\vspace{-1.5mm}
\begin{mydef}Given a task type $\tau$, a team $K$ and an assignment $\eta$, the competence degree of the team to perform the task is defined as:
\begin{equation}
\label{eq:uprof}
u_{\mathit{prof}}(\eta) = 1-(\upsilon \cdot u(\eta)+(1-\upsilon) \cdot o(\eta))
\end{equation}
where $\upsilon \in [0,1]$ is the penalty given to the undercompetence of team $K$. 
\end{mydef}
\vspace{-1.5mm}
Notice that the larger the value of $\upsilon$ the higher the importance of the competence degree of team $K$, while the lower the value $\upsilon$, the less important its undercompetence. The intuition here is that we might want to penalize more the undercompetency of teams, as some tasks strictly require teams to be at least as competent as defined in the task type.
\vspace{-1.5mm}
\begin{proposition}
For any $\eta$,  $u(\eta) + o(\eta) \in [0,1]$.
\label{prop1}
\end{proposition}

\begin{proof}
Given that (1) $l^{a}(c_{i}) \in [0,1]$ and $l_{i} \in [0,1]$; 
(2) If $\min(l^{a}(c_{i}) - l_{i},0)<0$ then $\max(l^{a}(c_{i}) -l_{i},0) = 0$; and
(3) If $\max(l^{a}(c_{i})-l_{i},0) > 0$ then $\min(l^{a}(c_{i}) - l_{i},0)=0$. Thus, from (1--3) 
we have
$|\min(l^{a}(c_{i}) - l_{i},0)|$ + $\max(l^{a}(c_{i})-l_{i},0) \in [0,1]$.
Let $n=|\{a \in \delta(c_{i})|l^{a}(c_{i})-l_{i} > 0\}|$, then obviously it holds that
$\frac{n \cdot (|\min(l^{a}(c_{i}) - l_{i},0)| + \max(l^{a}(c_{i})-l_{i},0))}{n} \in [0,1]$ and as $|\delta(c_i)| \leq n$ then
$\frac{\sum_{a \in \delta(c_{i})}(|\min(l^{a}(c_{i}) - l_{i},0)| + \max(l^{a}(c_{i})-l_{i},0))}{n} \in [0,1]$  holds; and 
since $\sum_{i \in I_{\tau}} w_i = 1$ then \\
$\sum_{i \in I_{\tau}} w_i \cdot \frac{\sum_{a \in \delta(c_{i})}(|\min(l^{a}(c_{i}) - l_{i},0)| + \max(l^{a}(c_{i})-l_{i},0))}{n} \in [0,1]$;
Finally, distributing, this equation is equivalent to: \\
$\sum_{i \in I_{\tau}} w_i \frac{\sum_{a \in \delta(c_{i})}(|\min(l^{a}(c_{i}) - l_{i},0)|}{n} \\
+ \sum_{i \in I_{\tau}} w_i \frac{\sum_{a \in \delta(c_{i})}(\max(l^{a}(c_{i})-l_{i},0))}{n} \in [0,1]$ which in turn is equivalent to $ u(\eta) + o(\eta) \in [0,1]$.
\end{proof}
\vspace{-1.5mm}
Function $u_{\mathit{prof}}$ is used to measure how proficient a team is for a given task assignment. However, counting on the required competences to perform a task does not guarantee that the team will succeed at performing it. Therefore, in the next subsection we present an evaluation function to measure \emph{congeniality} within teams. Unlike our measure for proficiency, which is based on considering a particular task assignment, our congeniality measure will solely rely on the personalities and genders of the members of a team. 
\subsection{Evaluating team congeniality} \label{ssec:con} 
\label{ssec:congeniality}

Inspired by the experiments of Douglass J. Wilde \cite{Wilde2009} we will define the team utility function for congeniality $u_{con}(K)$, such that:
\begin{itemize}
\vspace{-1.5mm}
\item it values more teams whose SN and TF personality dimensions are as diverse as possible;
\vspace{-1.5mm}
\item it prefers teams with at least one agent with positive EI and TF dimensions and negative PJ dimension, namely an extrovert, thinking and judging agent (called ETJ personality),
\vspace{-1.5mm}
\item it values more teams with at least one introvert agent;
\vspace{-2.5mm}
\item it values gender balance in a team.
\end{itemize}
Therefore, the higher the value of function $u_{con}(K)$, the more diverse the team is. 
Formally, this team utility function is defined as follows:
\vspace{-1mm}
\begin{equation}
\label{eq:ucon}
\begin{aligned}
u_{con}(K) = & \sigma_{SN}(K) \cdot \sigma_{TF}(K) +  \max_{a_i \in K}{((0,\alpha, \alpha, \alpha) \cdot {\bf p_i}, 0)} \\ 
 & + {\max_{a_i \in K}{((0,0,-\beta,0) \cdot {\bf p_i}, 0)}} + \gamma \cdot \sin{(\pi \cdot g(K))}
\end{aligned}
\vspace{-2.5mm}
\end{equation}
where the different parameters are explained next. 
\begin{itemize}
\vspace{-1.5mm}
\item $\sigma_{SN}(K)$ and $\sigma_{TF}(K)$: These variances are computed over the SN and TF personality dimensions of the members of team $K$. Since we want to maximise $u_{con}$, we want these variances to be as large as possible. The larger the values of $\sigma_{SN}$ and $\sigma_{TF}$ the larger their product will be, and hence the larger team diversity too. 
\vspace{-4mm}
\item $\alpha$:  The maximum variance of any distribution over an interval $[a,b]$ corresponds to a distribution with the elements evenly situated at the extremes of the interval. The variance will always be $\sigma^2 \le ((b-a)/2)^2$. In our case with $b=1$ and $a=-1$ we have $\sigma \le 1$. Then, to make the four factors equally important and given that the maximum value for ${\bf p_i}$ (the personality profile vector of agent $a_i$) would be $(1, 1, 1, 1)$ a maximum value for $\alpha$ would be $3 \alpha = ((1-(-1))/2)^2 = 1$, as we have the factor $\sigma_{SN} \cdot \sigma_{TF}$, so $\alpha \le 0.33(3)$. For values situated in the middle of the interval the variance will be $\sigma^2 \le \frac{(b-a)^2}{12}$, hence a reasonable value for $\alpha$ would be $\alpha = \frac{\sqrt[]{(1-(-1))^2)/12}}{3} = 0.19$
\vspace{-1.5mm}
\item $\beta$: A similar reasoning shows that $\beta \le 1$.
\vspace{-1.5mm}
\item $\gamma$ is a parameter to weigh the importance of a gender balance and $g(K) = \frac{w(K)}{w(K) + m(K)}$. Notice that for a perfectly gender balanced team with $w(K) = m(K)$ we have that
$\sin{(\pi \cdot g(K))} = 1$. The higher the value of $\gamma$, the more important is that team $u_{con}$ is gender balanced. Similarly to reasoning about $\alpha$ and $\beta$, we assess $\gamma \leq 1$. In order to make this factor less important than the others in the equation we experimentally assessed that $\gamma = 0.1$ is a good compromise.
\end{itemize}
\vspace{-1.5mm}
In summary, we will use a utility function $u_{con}$ such that: $\alpha = \frac{\sigma_{SN}(K) \cdot \sigma_{TF}(SK)} 3$, $\beta = 3 \cdot \alpha $ and $\gamma = 0.1$. 

\subsection{Evaluating synergistic teams}

Depending on the task type, different importance for congeniality and proficiency should be given. For instance, creative tasks require a high level of communication and exchange of ideas, and hence, teams require a certain level of congeniality. While, repetitive tasks require good proficiency and less communication. The importance of proficiency ($\lambda$) and congeniality ($\mu$) is therefore a fundamental aspect of the task type. Now, given a team, we can combine its competence value  (in equation \ref{eq:uprof}) with its congeniality value (in equation \ref{eq:ucon}) to measure its \emph{synergistic value}. 
\vspace{-1.5mm}
\begin{mydef}
Given a team $K$, a task type $\tau = \linebreak \langle \lambda, \mu, {\{(c_{i},l_{i}, w_{i})\}_{i \in I_{\tau}}} \rangle$ and a task assignment $\eta: K \rightarrow 2^{C_{\tau}}$, the synergistic value of team $K$ is defined as:
\vspace{-1.5mm}
\begin{equation}
s(K,\eta) = \lambda \cdot u_{\mathit{prof}}(\eta) + \mu \cdot u_{con}(K)
\end{equation}
where $\lambda \in [0,1]$ is the grade to which the proficiency of team $K$ is important, and $\mu \in [-1,1]$ is the grade to which the task requires diverse personalities.
\end{mydef}

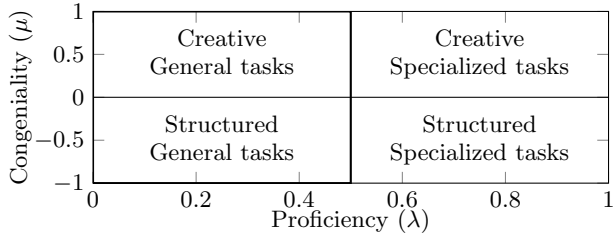
\begin{figure}
\caption{Values of congeniality and proficiency with respect to the task type.}
\begin{tikzpicture}
\begin{axis}[
    axis line style={->},
    x label style={at={(axis description cs:0.5,-0.1)},anchor=north},
    y label style={at={(axis description cs:-0.1,.5)},anchor=south},
  xlabel=Proficiency ($\lambda$),
  ylabel=Congeniality ($\mu$),
  xmin=0,
  xmax=1,
  ymin=-1,
  ymax=1,
  unit vector ratio=6 1,
]
    \node[black] at (axis cs:0.25,0.5) {
    \begin{tabular}{c}
    Creative \\ General tasks
  \end{tabular}};
        \node[black] at (axis cs:0.25,-0.5) {\begin{tabular}{c}
    Structured \\ General tasks
  \end{tabular}};
        \node[black] at (axis cs:0.75,0.5) {\begin{tabular}{c}
    Creative \\ Specialized tasks
  \end{tabular}};
        \node[black] at (axis cs:0.75,-0.5) {\begin{tabular}{c}
    Structured \\ Specialized tasks
  \end{tabular}};
        
    \draw [black, thick] (axis cs:0,-1) rectangle (axis cs:0.5,1);
	\draw (0,0) -- (1,0);
\end{axis}
\end{tikzpicture}
\label{tbl:parameters}
\vspace{-6mm}
\end{figure}

Figure \ref{tbl:parameters} shows the relation between the parameters $\lambda$ and $\mu$. 
In general, the higher the $\lambda$, the higher importance is given to the proficiency of a team. The higher the $\mu$ the more important is personality diversity. Notice, that the $\mu$ can be lower than zero. Having $\mu$ negative, we impose that the congeniality value will be as low as possible (to maximize $s(K,\eta)$) and so, team homogeneity is preferred. This situation may happen while performing tasks in unconventional performance environments that have serious consequences associated with failure. In order to quickly resolve issues, a team needs to be proficient and have team-mates who understand one another with minimum communication cost (which is associated to homogeneity of a team). 

\subsection{The synergistic team composition problem}
\label{ssec:synergisticProblem}

In what follows we consider that there are multiple instances of the same task to perform. Given a set of agents $A$, our goal is to split them into teams so that each team, and the whole partition of agents into teams, is balanced in terms of competences, personality and gender. 
We shall refer to these balanced teams as \emph{synergistic teams}, meaning that they are both congenial and proficient. 

Therefore, we can regard our team composition problem as a particular type of set partition problem. We will refer to any partition of $A$ as a team partition. However, we are interested in a particular type of team partitions, namely those where teams are constrained by size $m$ as follows.

\begin{mydef}
Given a set of agents $A$, we say that a team partition $P_m$ of $A$ is constrained by size $m$ iff: (i) for every team $K_i \in P_m$, $K_i \in \cal{K}_A$, $\max(m-1, 2) \leq |K| \leq m+1$ holds; and (ii) for every pair of teams $K_i, K_j \in P_m$ $||K_i| - |K_j|| \le 1$. 
\end{mydef}

As $|K| / m$ is not necessarily a natural number, we may need to allow for some flexibility in team size within a partition. This is why we introduced above the condition $\max(m-1, 2) \leq |K| \leq m+1$. In practical terms, in a partition we may have teams differing by one agent. We note by ${\cal P}_m(A)$ the set of all team partitions of $A$ constrained by size $m$. Henceforth, we will focus on team partitions constrained by some size. Since our goal is to find the most competence-balanced and psychologically-balanced team partition, we need a way to measure the synergistic value of a team partition, which we define as follows: 

\begin{mydef}
Given a task $t = \langle \tau, m \rangle$, a team partition $P_m$ and an assignment $\eta_i$ for each team $K_i \in P_m$, the synergistic value of $P_m$ is computed by:
\vspace{-1.5mm}
\begin{equation}
u(P_m,\bm{\eta}) = \prod_{i =1}^{|P_m|} s(K_i,\eta_i)
\end{equation}
\vspace{-1.5mm}
where $\bm{\eta}$ stands for the vector of task assignments $\eta_1,\ldots, \linebreak \eta_{|P_m|}$.
\end{mydef}

Notice that the use of a Bernoulli-Nash function over the synergistic values of teams will favour team partitions whose synergistic values are balanced.

Now we are ready to cast the synergistic team composition problem as the following optimisation problem:

\begin{mydef}
Given task $t = \langle \tau, m \rangle$ and set of agents $A$ the \textbf{synergistic team formation problem (STFP)} is the problem of finding a team partition constrained by size $m$, together with competence assignment for its teams, whose synergistic value is maximal. Formally, the STFP is the problem of finding the partition in $P \in \mathcal{P}_m(A)$ and the task assignments $\bm{\eta}$ for the teams in $P_m$ that maximises  $u(P_m,\bm{\eta})$.
\end{mydef}

\vspace{-2mm}
\section{Solving STFP}\label{sec:TeamForm}
In this section we detail an algorithm, the so-called \emph{\SynTeam}, which solves the synergistic team formation problem described above. We will start from describing how to split agents into a partition (see subsection \ref{ssec:dist}). Next, we will move on to the problem of assigning competences in a task to team members (see subsection \ref{ssec:asg}), so that the utility of synergistic function is maximal. Finally, we will explain \emph{\SynTeam} that is a greedy algorithm that quickly finds a first, local solution, to subsequently improve it, hoping to reach a global optimum.

\subsection{How do we split agents?} \label{ssec:dist}

We note by $n = |A|$ the number of agents in $A$, by $m \in \mathbb{N}$ the target number of agents in each team, and by $b$ the minimum total number of teams, $b = \left\lfloor  n/m\right\rfloor$. We define the quantity distribution of agents in teams of a partition, noted $T: \mathbb{N} \times \mathbb{N} \to \mathbb{N} \times \mathbb{N} \cup (\mathbb{N} \times \mathbb{N})^2 $ as:
\vspace{-2mm}
\begin{equation}
\begin{multlined}
T(n,m) = \\
\begin{cases}
\{(b, m)\}  & \text{if  } n \geq m \textit{ and } n \bmod m  = 0
\\
 \{(n \bmod m,m + 1), \\(b - (n \bmod m),m)\}
  & \text{if  } n \geq m \textit{ and } n \bmod m  \le b
\\
\{(b, m),(1, n \bmod m)\} & \text{if  } n \geq m \textit{ and } n \bmod m > b
\\
\{(0,m)\} & \text{otherwise}
\end{cases}
\end{multlined}
\end{equation}

Note that depending on the cardinality of $A$ and the desired team size, the number of agents in each team may vary by one individual (for instance if there are $n=7$  agents in $A$ and we want to compose duets ($m=2$), we split agents into two duets and one triplet).

\subsection{Solving an Assignment} \label{ssec:asg}


There are different methods to build an assignment. We have decided to solve our assignment problem by using the minimum cost flow model \cite{ahuja1993network}. This is one of the most fundamental problems within network flow theory and it can be efficiently solved. For instance, in \cite{orlin1993faster}, it was proven that the minimum cost flow problem can be solved in $O(m \cdot log(n) \cdot (m + n \cdot log(n)))$ time with $n$ nodes and $m$ arcs.

Our problem is as follows: 
There are a number of agents in team $K$ and a number of competence requests in task $t$. Any agent can be assigned to any competence, incurring some cost that varies depending on the agent competence level of the assigned competence. We want to get each competence assigned to at least one agent and each agent assigned to at least one competence in such a way that the total cost (that is both undercompetence and overcompetence) of the assignment is minimal with respect to all such assignments. 

Formally, let $G = (N, E)$ be a directed network defined by a set $N$ of $n$ nodes and a set $E$ of $e$ directed arcs. There are four types of nodes: (1) one source node; (2) $|K|$ nodes that represent agents in team $K$; (3) $|C_{\tau}|$ competence requests that form task type $\tau$; and (4) one sink node. Each $arc$ $(i, j) \in E$ has an associated cost $p_{ij} \in \mathbb{R}^+$ that denotes the cost per unit flow on that $arc$.  We also associate with each $arc$ $(i, j) \in E$ a capacity $u_{ij} \in \mathbb{R}^+$ that denotes the maximum amount that can flow on the arc. In particular, we have three kinds of edges: (1) Supply arcs. These edges connect the source to agent nodes. Each of these arcs has zero cost and a positive capacity $u_{ij}$ which define how many competences at most can be assigned to each agent. (2) Transportation arcs. These are used to ship supplies. Every transportation edge $(i, j) \in E$ is associated with a shipment cost $p_{ij}$ that is equal to:
\begin{equation*}
p_{ij} =
\begin{cases}
(l^{a_i}(c_{\mathit{j}}) - l_{\mathit{j}}) \cdot (1-\upsilon) \cdot w_{\mathit{j}} & \text{if  } l^{a_i}(c_{\mathit{j}} - l_{\mathit{j}}) > 0\\
-(l^{a_i}(c_{\mathit{j}}) - l_{\mathit{j}}) \cdot \upsilon \cdot w_{\mathit{j}} & \text{if  } l^{a_i}(c_{\mathit{j}} - l_{\mathit{j}}) < 0
\end{cases}
\label{costeq}
\end{equation*}
\noindent
where $v \in [0,1]$ is the penalty given to the undercompetence of team $K$(see subsection \ref{ssec:prof} for the definition). 
(3) Demand arcs. These arcs connect the competence requests nodes to the sink node. These arcs have zero costs and positive capacities $u_{ij}$ which equal the demand for each competence. 

Thus, a network is denoted by $(G, w, u, b)$. We associate with each node $i \in N$ an integer number $b(i)$ representing its supply. If $b(n) > 0$ then $n$ is a source node, if $b(n) < 0$ then $n$ is a sink node. In order to solve a task assignment problem, we use the implementation of \cite{goldberg1990finding} provided in the ort-tools.\footnote{\url{https://github.com/google/or-tools/blob/master/src/graph/min_cost_flow.h}} 
\vspace{-2mm}
\begin{figure}
\includegraphics[max size={\textwidth}{10.35cm}]{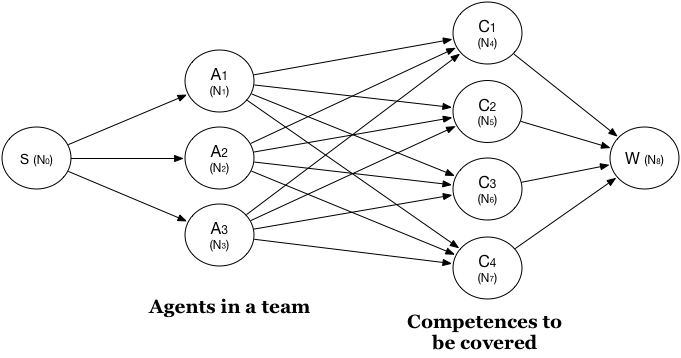}
\caption{An example of an assignment graph $G(N,E)$}\label{asg}
\vspace{-6mm}
\end{figure}

\paragraph{Example} Let us consider a team of three agents $K = \{a_1, a_2, a_3\}$:
\begin{itemize}
\vspace{-1.5mm}
\item $a_1 = \langle id_1, `woman', p_1, [l(c_1) = 0.9, l(c_2) = 0.5]\rangle$
\vspace{-1.5mm}
\item $a_2 = \langle id_2, `man', p_2, [l(c_2) = 0.2, l(c_3) = 0.8]\rangle$
\vspace{-1.5mm}
\item $a_3 = \langle id_3, `man', p_3, [l(c_2) = 0.4, l(c_4) = 0.6]\rangle$
\end{itemize}
and task type $\tau$ containing four competence requests \\ $\{(c_{1},0.8, 0.25), (c_{2}, 0.6, 0.25), (c_{3},0.6, 0.25),(c_{4},0.6, 0.25)\}$. \\ The penalty given to undercompetence is equal to $\upsilon=0.6$.

Our goal is to assign agents to competence requests, so that: (1) every agent is responsible for at least one competence, (2) every competence is covered by at least one agent, (3) the overall ``cost'' in minimal. 
As shown in figure \ref{asg}, we build a graph out of $n = 9$ nodes that is: one source node ($N_0$), three agents nodes ($N_1 - N_3$), four competences nodes ($N_4 - N_7$) and a sink node ($N_8$). Next, we add edges: (1) between source node $N_0$ and all agent nodes $N_1 - N_3$ that have a cost $p_{si} = 0$ and capacity $u_{si} = 2$ for all $i$ as the maximum number of competences assigned to one agent cannot be bigger than two if we want to make sure that all agents are assigned to at least one competence; (2) between agent nodes $N_1 - N_3$ and competence nodes ($N_4 - N_7$), where each capacity $u_{ij} = 1$ and we calculate costs according to the equation \ref{costeq}. For instance, the cost between $N_1$ and $N_4$ is equal to: $(0.9 - 0.8) \cdot (1-0.6) \cdot 0.25 = 0.01$. We multiply all costs by $1000$ to meet the requirements of the solver (edges need to be integer). Hence, the final cost $p_{14}=10$; (3) edges between competence nodes $N_4 - N_7$ and sink node $N_8$ that have costs $p_{jw} = 0$ and capacities $u_{jw} = 1$ to impose that each is assigned.
Once the graph is built, we pass it to the solver to get the assignment, and we get $c_1$ and $c_2$ assigned to $a_1$, $c_3$ assigned to $a_2$ and $c_4$ assigned to $a_3$.

\subsection{SynTeam algorithm} \label{ssec:SynTeam} 

Algorithm \ref{alg:teamDistribution} shows the SynTeam pseudocode.
Algorithm \ref{alg:teamDistribution} is divided into two parts:

{\bf 1. \textsl{Find a first team partition}}. This part of the algorithm simply builds a partition by randomly assigning agents to teams of particular team sizes. This part goes as follows. Given a list of agents $A$, we start by shuffling the list so that the order of agents in the list is random (line~1). Next, we determine the quantitative distribution of individuals among teams of size $m$ using function $T(|A|,m)$ as defined in section \ref{ssec:dist} (line~2). We start from the top of the shuffled list of agents (line~3). For each number of teams (line~4), we define a temporary set $team$ to store a current team (line~5). We add to $team$ subsequent $size$ agents from the shuffled list of agents (line~7). We add the newly created team to the team partition $P_{\mathit{best}}$ that we intend to build (line~10). When reaching line~14, $P_{\mathit{best}}$ will contain a first disjoint subset of teams (a team partition). 

{\bf 2. \textsl{Improve the current best team partition}}. The second part of the algorithm consists in improving the current best team partition. The idea is to obtain a better team partition by performing crossovers of two randomly selected teams to yield two better teams. In this part, we took inspiration from simulated annealing methods, where the algorithm might accept swaps that actually decrease the solution quality with a certain probability. The probability of accepting worse solutions slowly decreases as the algorithm explores the solution space (as the number of iterations increases). The annealing schedule is defined by the $\mathit{cooling\_rate}$ parameter. We have modified this method to store the partition with the highest synergistic evaluation found so far.
In detail, the second part works as follows. First, we select two random teams, $K_1$ and $K_2$, in the current team partition (line~15). Then we compute all team partitions of size $m$ with agents in $K_1 \cup K_2$ (line~19), and we select the best candidate team partition, named $P_{\mathit{bestCandidate}}$ (lines~19~to~26). If the best candidate synergistic utility is larger than the utility contribution of $K_1$ and $K_2$ to the current best partition $P_{\mathit{best}}$ (line~27), then we replace teams $K_1$ and $K_2$ by the teams in the best candidate team partition (line~28). If the best candidate team partition utility is lower
, then we check if the probability of accepting a worse solution is higher than a uniformly sampled value from $[0,1]$ (line~29). 
If so, 
we replace teams $K_1$ and $K_2$ by the teams in the best candidate team partition (line~30) and we lower $heat$ by a cooling rate. This part of the algorithm continues until the value of $heat$ reaches $1$ (line~13). We also store the best partition found so far (line~34) to make sure we do not end up with worse solution. Finally, we return found best partition $P_{\mathit{bestEver}}$ as well as the assignment $\eta$ for each team.
\begin{algorithm}[h]
\small
\caption{\quad \SynTeam}
\label{alg:teamDistribution}
\begin{algorithmic}[1]
    \Require $A$ \Comment{The list of agents}
	\Require $T(|A|,m)$ \Comment{Quantitative team distribution}
    \Require $P_{\mathit{best}} = \emptyset$ \Comment{Initialize best partition}
     \Require $\mathit{heat=10}$ \Comment{Initial temperature for second step}
     \Require $\mathit{Cooling\_rate}$ \Comment{Heating decrease}
     \Ensure $(P, \bm{\eta})$ \Comment{Best partition found and best assignments}
    \State $\mathit{random.shuffle(A)}$
    \If {$T(|A|,m) \ne (0,m)$}
    \State $\mathit{index} = 0$ \Comment{Used to iterate over the agent list}
    \ForAll{$(\mathit{numberOfTeams}, \mathit{size)} \in T(|A|,m)$}
      \State $team = \emptyset$
      \For {$i \in (0,\dots ,\mathit{(size-1))}$}
      \State $team = team \cup A[\mathit{index}]$
      \State $\mathit{index}=\mathit{index} + 1$
      \EndFor
      \State $P_{\mathit{best}} = P_{\mathit{best}} \cup \{team\}$
          \EndFor
              \State $\bm{ \eta_{\mathit{best}}} = \mathit{assign\_agents}(P_{\mathit{best}})$ \Comment{see Subsection \ref{ssec:asg}}
          \State $(P_{\mathit{bestEver}}, \mathit{bestValueEver}) = (P_{\mathit{best}},u(P_{\mathit{best}},\bm{ \eta_{\mathit{best}}}))$
    \While{$\mathit{heat} > 1$} 
    \State $(K_1,K_2) = selectRandomTeams(P_{\mathit{best}}$)
        \State $(\eta_1,\eta_2) = \mathit{assign\_agents}(\{K_1,K_2\})$
     \State $\mathit{contrValue} = u(\{K_1,K_2\},(\eta_1,\eta_2))$
    \State $(P_{\mathit{bestCandidate}}, \mathit{best Candidatevalue}) = (\emptyset,0)$
    \ForAll {$P_{\mathit{candidate}} \in P_m(K_1 \cup K_2) \setminus \{K_1,K_2\}$}
    \State $(\eta_1,\eta_2) = assign\_agents(P_{\mathit{candidate}})$ 
    \State $\mathit{candidateValue} = u(P_{\mathit{candidate}},(\eta_1,\eta_2))$
    \If{$\mathit{candidateValue} > \mathit{bestCandidateValue}$}
    \State $P_{\mathit{bestCandidate}} = P_{\mathit{candidate}}$
    \State $\mathit{bestCandidateValue} = \mathit{candidateValue}$
    \EndIf
    \EndFor
   \If{$\mathit{bestCandidateValue} > \mathit{contrValue}$}
   \State $P_{\mathit{best}} = replace(\{K_1,K_2\},P_{\mathit{bestCandidate}}, P_{\mathit{best}})$
   \ElsIf{$\mathbb{P}(\mathit{bestCandidateValue}, \mathit{contrValue}, heat)$ \StatexIndent[2] $\geq \mathit{random}(0, 1)$}
   \State $P_{\mathit{best}} = replace(\{K_1,K_2\},P_{\mathit{bestCandidate}},P_{\mathit{best}})$
    \EndIf
    \State $\bm{ \eta_{\mathit{best}}} = \mathit{assign\_agents}(P_{\mathit{best}})$
    \If {$\mathit{bestValueEver} < u(P_{\mathit{best}},\bm{ \eta_{\mathit{best}}})$}
    \State $P_{\mathit{bestEver}} = P_{\mathit{best}}$
   \EndIf
      	\State $heat$ = $heat-\mathit{Cooling\_rate}$
    \EndWhile
    \State $return(P_{\mathit{bestEver}},\mathit{assign\_agents(P_{\mathit{bestEver}}}))$
    \EndIf
\end{algorithmic}
\end{algorithm}
\vspace{-4mm}
\section{Experimental Results} \label{sec:results}

\subsection{Experimental Setting}
``Institut Torras i Bages'' is a state school near Barcelona. Collaborative work has been implemented there for the last 5 years in their final assignment (``Treball de S\'{\i}ntesi'') with a steady and significant increase in the scores and quality of the final product that students are asked to deliver. This assignment takes one week and is designed to check if students have achieved, and to what extent, the objectives set in the various curricular areas. It is a work that encourages teamwork, research, and tests relationships with the environment. Students work in teams and at the end of every activity present their work in front of a panel of teachers that assess the content, presentation and cooperation between team members. This is a creative task, although requiring high level of competences. 
\subsection{Data Collection} 
In current school practice, teachers group students according to their own, manual method based on the knowledge about students, their competences, background and social situation. This year we have used our grouping system based only on personality (\SynTeam\ with $\lambda = 0, \mu = 1$) upon two groups of students: `3r ESO A' (24 students), and `3r ESO C' (24 students). Using computers and/or mobile phones, students answered the questionnaire (described in section \ref{pers}) which allowed us to divide them into teams of size three for each class. Tutors have evaluated each team in each partition giving an integer value $v \in [1,10]$ meaning their expectation of the performance of each team. 
Each student team was asked to undertake the set of interdisciplinary activities (``Treball de S\'{\i}ntesi'') described above. We have collected each student's final mark for ``Treball de S\'{\i}ntesi'' as well as final marks obtained for all subjects. That is:  Catalan, Spanish, English, Nature, Physics and Chemistry, Social Science, Math, Physical Education, Plastic Arts, Technology. We have used a matrix provided by the tutors to relate each subject to different kinds of intelligence (that in education are understood as competences) needed for this subject. There are eight types of human intelligence \cite{gardner1987theory}, each representing different ways of processing information: Naturalist, Interpersonal, Logical/Mathematical, Visual/Spatial, Body/Kinaesthetic, Musical, Intrapersonal and Verbal/Linguistic. This matrix for each subject and each intelligence is shown in figure \ref{matrix}.

\begin{figure}[h]
\centering
$\begin{bmatrix}
  0 & 1 & 0 & 0 & 0 & 0 & 1 & 1 
\\0 & 1 & 0 & 1 & 0 & 1 & 1 & 1 
\\0 & 1 & 0 & 0 & 0 & 1 & 1 & 1  
\\1 & 1 & 0 & 1 & 1 & 0 & 1 & 1  
\\1 & 1 & 1 & 1 & 0 & 0 & 1 & 1  
\\1 & 1 & 0 & 0 & 0 & 0 & 1 & 1  
\\0 & 1 & 1 & 1 & 0 & 0 & 1 & 1  
\\0 & 1 & 0 & 1 & 1 & 0 & 1 & 1  
\\0 & 1 & 0 & 1 & 1 & 0 & 1 & 0  
\\1 & 1 & 1 & 0 & 1 & 0 & 1 & 1  
\end{bmatrix}$
\label{matrix}
\caption{Matrix matching Intelligence with subjects (each row corresponds to a subject, each column to an intelligence)}
\end{figure}

\noindent Subjects are represented by rows and intelligences by columns of the matrix in the order as provided above. Based on this matrix we calculate values of intelligences for every student by averaging all values obtained by her that are relevant for this intelligence. For instance, for Body/Kinaesthetic intelligence, we calculate an average of student marks obtained in Nature, Physical Education, Plastic Arts and Technology. An alternative way to measure students' competences level can be by calculating the collective assessments of each competence (like proposed by \cite{andrejczukCompetences}).

Finally, having competences (Intelligences), personality and actual performance of all students, we are able to calculate synergistic values for each team. We also calculate the average of marks obtained by every student in a team to get teams' performance values.

\subsection{Results}
\noindent 
Given several team composition methods, we are interested in comparing them to know which method better predicts team performance. Hence, we generate several team rankings using the evaluation values obtained through different methods. First, we generate a ranking based on actual team performance that will be our base to compare other rankings. Second, we generate a ranking based on the expert evaluations. Finally,  we generate several rankings based on calculated synergistic values with varying importance of congeniality and proficiency. Since ``Traball de S\'{\i}ntesi'' is a creative task, we want to examine the evaluation function with parameters $\mu > 0$ and $\lambda = 1-\mu$. In particular, we want to observe how the rankings change when increasing the importance of competences. 
Notice that teacher and actual performance rankings may include ties since the pool of possible marks is discrete (which is highly improbable in case of \SynTeam\ rankings). Therefore, before generating rankings based on synergistic values, we round them up to two digits to discretize the evaluation space. An ordering with ties is also known as a \emph{partial ranking}. 

Next, we compare teacher and \SynTeam\ rankings with the actual performance ranking using the standardized Kendall Tau distance. For implementation details, refer to the work by Fagin et al. \cite{Fagin:2004:CAR,fagin2006comparing}, which also provide sound mathematical principles to compare partial rankings. The results of the comparison are shown in Figure \ref{asg}. Notice that the lower the value of Kendall Tau, the more similar the rankings. We observe that the \SynTeam\ ranking improves as the importance of competences increases, and it is best at predicting students' performance for $\lambda = 0.8$ and $\mu = 0.2$ (Kendall Tau equal to $0.15$).  A standardised Kendall Tau distance for teacher ranking is equal to $0.28$, which shows that \SynTeam\ predicts the performance better than teachers, when competences are included ($\lambda > 0.2$). We also calculate the values of Kendall Tau for random ($0.42$) and reversed ($0.9$) rankings to benchmark teacher and \SynTeam\ grouping methods. The results show that both teachers and \SynTeam\ are better at predicting students' performance than the random method. 

\begin{figure}
\includegraphics[max size={\textwidth}{10.35cm}]{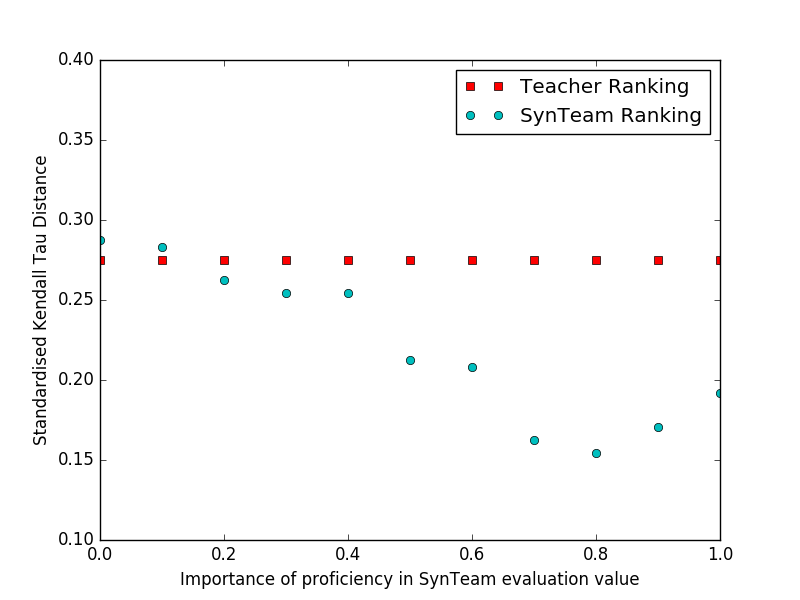}
\caption{Comparison of Kendall-Tau distances between different methods.}\vspace{-2mm}
\label{asg}
\vspace{-2mm}
\end{figure}

\section{Discussion} \label{sec:discuss}
In this paper we introduced \SynTeam, an algorithm for partitioning groups of humans into competent, gender and psychologically balanced teams. 

To our knowledge, \SynTeam\ is the first computational model to build synergistic teams that not only work well together, but  are also competent enough to perform an assignment requiring particular expertise. 

We have decided to evaluate our algorithm in the context of a classroom. Besides obvious advantages of observing students work in person, this scenario gave us an opportunity to compare our results with real-life, currently used practice. The results show that \SynTeam\ is able to predict team performance better that the experts that know the students, their social background, competences, and cognitive capabilities. 

The algorithm is potentially useful for any organisation that faces the need to optimise their problem solving teams (e.g. a classroom, a company, a research unit). The algorithm composes teams in a purely automatic way without consulting experts, which is a huge advantage for environments where there is a lack of experts.


Regarding future work, We would like to investigate how to determine quality guarantees of the algorithm. 

Additionally, there is a need to consider richer and more sophisticated models to capture the various factors that influence the team composition process in the real world. We will consider how our problem relates to the constrained coalition formation framework \cite{Rahwan}. This may help add constraints and preferences coming from experts that cannot be established by any algorithm, e.g. Anna cannot be in the same team with Jos\'e as they used to have a romantic relationship.

\newpage
\bibliographystyle{plain}

\end{document}